\documentclass{article}

\usepackage{microtype}
\usepackage{graphicx}
\usepackage{subfigure}
\usepackage{booktabs} 
\usepackage{natbib}
\usepackage{parskip}

\usepackage{hyperref}

\usepackage{amsmath}
\usepackage{amssymb}
\usepackage{mathtools}
\usepackage{bbm}
\usepackage{amsthm}
\usepackage{physics}
\usepackage{xcolor}        
\usepackage{thmtools}
\usepackage{thm-restate}

\usepackage[capitalize,noabbrev]{cleveref}
\usepackage[left=3cm,right=3cm]{geometry}
\usepackage{fullpage}

\definecolor{mydarkblue}{rgb}{0,0.08,0.45}
\hypersetup{ %
    colorlinks=true,
    linkcolor=mydarkblue,
    citecolor=mydarkblue,
    filecolor=mydarkblue,
    urlcolor=mydarkblue,
}

\theoremstyle{plain}
\newtheorem{theorem}{Theorem}[section]

\theoremstyle{definition}

\theoremstyle{remark}
\newtheorem{remark}[theorem]{Remark}

\DeclareMathOperator*{\argmax}{argmax}

\newcommand{\Uni}{\mathrm{Uni}}
\newcommand{\Gauss}{\mathcal{N}}

\newcommand{\E}{\mathbb{E}}
\newcommand{\Var}{\mathrm{Var}}
\newcommand{\Prob}{\mathbb{P}}
\newcommand{\R}{\mathbb{R}}
\newcommand{\N}{\mathbb{N}}
\newcommand{\supp}{\mathrm{supp}}
\newcommand{\Ind}{\mathbbm{1}}

\newcommand{\ECE}{\mathrm{ECE}}
\newcommand{\SECE}{\mathrm{LS}\text{-}\mathrm{ECE}}
\newcommand{\TV}{\mathrm{TV}}
\newcommand{\tvto}{\stackrel{\TV}{\to}}

\newcommand{\BIN}{\mathrm{BIN}}
\newcommand{\smECE}{\textsf{smECE}}
\newcommand{\smECEp}{\widetilde{\smECE}}

\newcommand{\de}[0]{\delta}
\newcommand{\De}[0]{\Delta}
\newcommand{\ep}[0]{\varepsilon}
\newcommand{\rh}[0]{\rho}
\newcommand{\si}[0]{\sigma}
\newcommand{\iy}[0]{\infty}
\newcommand{\ab}[1]{\left|#1\right|}
\newcommand{\pa}[1]{\left(#1\right)}
\newcommand{\ba}[1]{\left[#1\right]}

\newcommand{\dr}[2]{\frac{d#1}{d#2}}
\newcommand{\pf}[2]{\left(\frac{#1}{#2}\right)}

\newcommand{\fc}[2]{\frac{#1}{#2}}
\newcommand{\rc}[1]{\frac{1}{#1}}

\newcommand\numberthis{\addtocounter{equation}{1}\tag{\theequation}}
\numberwithin{equation}{section}

\usepackage[textsize=tiny]{todonotes}

\title{How Flawed Is ECE? \\ An Analysis via Logit Smoothing}
\author{
\begin{tabular}{c c}
  \begin{tabular}{c}
    Muthu Chidambaram$^*$ \\
    Duke University \\
    \texttt{muthu@cs.duke.edu}
  \end{tabular} & \qquad
  \begin{tabular}{c}
    Holden Lee$^*$ \\
    Johns Hopkins University \\
    \texttt{hlee283@jhu.edu}
  \end{tabular} \\ \\
  \begin{tabular}{c}
    Colin McSwiggen$^*$ \\
    New York University \\
    \texttt{csm482@nyu.edu}
  \end{tabular} & \qquad
  \begin{tabular}{c}
    Semon Rezchikov\\
    Princeton University \\
    \texttt{semonr@princeton.edu}
  \end{tabular}
\end{tabular}
}
\date{\today}

\begin{document}

\maketitle
\def\thefootnote{*}\footnotetext{Equal contribution. A version of this paper appears in the \textit{Proceedings of the
$\mathit{41}^{st}$ International Conference on Machine Learning},
Vienna, Austria. PMLR 235, 2024.}\def\thefootnote{\arabic{footnote}}

\begin{abstract}
Informally, a model is calibrated if its predictions are correct with a probability that matches the confidence of the prediction. By far the most common method in the literature for measuring calibration is the expected calibration error (ECE). Recent work, however, has pointed out drawbacks of ECE, such as the fact that it is discontinuous in the space of predictors. In this work, we ask: how fundamental are these issues, and what are their impacts on existing results? Towards this end, we completely characterize the discontinuities of ECE with respect to general probability measures on Polish spaces. We then use the nature of these discontinuities to motivate a novel \textit{continuous, easily estimated} miscalibration metric, which we term \textit{Logit-Smoothed ECE (LS-ECE)}. By comparing the ECE and LS-ECE of pre-trained image classification models, we show in initial experiments that binned ECE closely tracks LS-ECE, indicating that the theoretical pathologies of ECE may be avoidable in practice. 
\end{abstract}

\section{Introduction}
The prevalence of machine learning across domains has increased drastically over the past few years, spurred by significant breakthroughs in deep learning for computer vision \citep{dalle2} and language modeling \citep{GPT3, openai2023gpt4, touvron2023llama}. Consequently, the underlying deep learning models are increasingly being evaluated for critical use cases such as predicting medical diagnoses \citep{elmarakeby2021biologically, nogales2021} and self-driving \citep{hu2023planningoriented}. In these latter cases, due to the risk associated with incorrect decision-making, it is crucial not only that the models be accurate, but also that they have proper predictive uncertainty.

This desideratum is formalized via the notion of \textit{calibration} \citep{dawid1982well, degroot1983comparison}, which codifies how well the model-predicted probabilities for events reflect their true frequencies conditional on the predictions. For example, in a medical context, a model that yields the correct diagnosis for a patient 95\% of the time when it predicts a probability of $\approx 0.95$ for that diagnosis can be considered to be calibrated.  

The analysis of whether modern deep learning models are calibrated can be traced back to the influential work of \citet{guo2017calibration}, which showed that recent models exhibit calibration issues not present in earlier models; in particular, they are overconfident when they are incorrect. These findings have been corroborated by a large body of subsequent work in which several training and post-training modifications have been proposed in order to improve calibration \citep{lakshminarayanan2017simple, kumar2018trainable, thulasidasan2019mixup, muller2020does, focalloss2021, wang2023metacalibration}.

However, the validity of these results depends on having an appropriate measure of calibration. The canonical measure of calibration in the machine learning literature has been the \textit{Expected Calibration Error (ECE)} and its binned variants \citep{naeini2014binary, nixon2019measuring}, and indeed all of the aforementioned works report ECE in some capacity.

Unfortunately, several works have pointed out (seemingly) significant drawbacks of ECE. First, it is discontinuous as a function of the model being considered. In other words, small changes to model predictions can cause large jumps in the ECE \citep{kakadeF08, FosterH18, blasiok2023unifying, blasiok2023smooth}. Second, it is not possible to efficiently estimate from samples \citep{arrieta2022metrics, lee2022t}, and binned variants can be sensitive to the choice of bin width \citep{nixon2019measuring, kumar2019verified, minderer2021revisiting}.

As a result of these drawbacks, a number of authors have recently proposed alternatives to ECE that enjoy better theoretical properties \citep{arrieta2022metrics, lee2022t, blasiok2023smooth, blasiok2023unifying}. Despite these proposals, as noted in \citet{blasiok2023smooth}, ECE continues to be the main metric reported in very recent studies. \citet{blasiok2023smooth} hypothesize that the reason for this fact is that ECE can be easily visualized and interpreted via reliability diagrams.

In this work, we propose an alternative explanation that may serve to justify the continued predominance of ECE: besides the fact that ECE is historically established and well supported by standard codebases, the pathologies of ECE are not encountered in practice due to noise inherent to the data and model training process. This paper formalizes one simple variant of such a noise model. In this model, we show that the addition of noise makes ECE continuous and leads to an effective estimation scheme, which moreover does not appreciably differ from direct estimates of ECE performed via binning.  

Informed by this perspective, we aim to answer the following questions in this work:

\begin{itemize}
\item Can we characterize the points of discontinuity of ECE?
\item Can these discontinuities be eliminated by a simple modification of the miscalibration metric? 
\item Does the discontinuous behavior of ECE actually pose a problem for estimating the calibration of real-world deep learning models?
\end{itemize}

\subsection{Summary of Main Contributions}
Our main contributions towards answering the above questions are as follows.
\begin{enumerate}
    \item In Section \ref{sec:cont}, we completely characterize the discontinuities of ECE in a very general setting.  We illustrate in detail how these considerations apply in the case of discrete data distributions with finite support, and we show that in this case the discontinuities are a measure zero set. The case of continuous distributions is more subtle, and intuitions from the discrete setting do not always carry over; however, we nevertheless provide a necessary and sufficient condition for discontinuity in the case of arbitrary distributions of data taking values in a Polish space, and we show that in this setting the ECE is always a lower semicontinuous functional.
    \item Building on the ideas of Section \ref{sec:cont}, we derive a modified ECE measure in Section \ref{sec:lsece} which we term \textit{Logit-Smoothed ECE (LS-ECE)}. We show that the LS-ECE is continuous in the space of predictors, for \textit{any} data distribution. Our results rely on establishing strong connections between convergence of the underlying joint probability measures in total variation on one hand, and continuity of the ECE functional on the other, which may be of independent interest.
    \item We further propose a consistent estimator of the LS-ECE in Section \ref{sec:est}, and show that our estimator can both be efficiently estimated and implemented. As an additional consequence of our estimation result, we show that LS-ECE can be used to produce a consistent estimator of the true ECE when the predictive distribution satisfies mild regularity conditions, which to the best of our knowledge is a stronger result than any pre-existing consistency results for ECE.
    \item Lastly, in Section \ref{sec:experiments}, we verify empirically that LS-ECE is continuous even when ECE is not, and also show that for the standard image classification benchmarks of CIFAR-10, CIFAR-100, and ImageNet, both ECE and LS-ECE produce near identical results across various models --- indicating that the theoretical pathologies of ECE may not pose an issue in practice.
\end{enumerate}
We note that, in the process of analyzing ECE, we have proposed yet another competing notion in the form of LS-ECE. We wish to stress that we are not trying to claim that LS-ECE is a ``better'' measure of calibration than recently proposed alternatives, and in fact it shares much in common with the SmoothECE proposal of \citet{blasiok2023smooth}. Rather, we view LS-ECE as a useful theoretical and empirical tool for sanity-checking ECE in a given setting --- our experiments in Section \ref{sec:imageclass} use it to suggest that reported ECE results may not be particularly brittle. Furthermore, we hope that the theoretical framework under which we formulate LS-ECE will prove useful in future analyses of calibration.

\subsection{Related Work}
\textbf{Estimation of ECE.} Perhaps the most common way to estimate ECE in the literature is by binning predictions into uniformly sized bins \citep{naeini2014binary}. A similarly popular approach is to bin predictions using equal mass bins, which leads to Adaptive Calibration Error \citep{nixon2019measuring}. These binning approaches are, however, known not to be consistent \citep{vaicenavicius2019evaluating}, and follow-up works have modified them further via debiasing schemes \citep{kumar2019verified, roelofs2022mitigating}. An alternative to binning is estimating ECE via kernel density/regression estimators \citep{brocker2008some, zhang2020mix, popordanoska2022consistent}, which trade off bin selection with bandwidth selection. 

\textbf{Alternatives to ECE.} Recent work has pursued several different directions for developing alternatives to ECE. These include, but are not limited to, proper scoring rules \citep{gneiting2007probabilistic, gneiting2007strictly}, estimating miscalibration using splines \citep{guptacalibration}, isotonic regression \citep{dimitriadis2021stable}, hypothesis tests for miscalibration \citep{lee2022t}, and cumulative plots comparing labels to predicted probabilities \citep{arrieta2022metrics}. \citet{blasiok2023unifying} detail several more alternatives to ECE, along with a theoretical framework based on distance to the nearest calibrated predictor that justifies the use of these alternatives in practice. Very recently, \citet{blasiok2023smooth} have proposed a kernel-smoothed ECE that satisfies the constraints of the framework of \citet{blasiok2023unifying} but still maintains the interpretability of ECE. The approach we take in this paper was developed independently and in parallel, and shares similarities with the approach of \citet{blasiok2023unifying} that we point out in Section \ref{sec:est}.

\section{Background}
\textbf{Notation.} Given $n \in \N$, we use $[n]$ to denote the set $\{1, 2, ..., n\}$. For a function $g: \R^n \to \R^m$ we use $g^i(x)$ to denote the $i^{\text{th}}$ coordinate function of $g$. For a probability distribution $\pi$ we use $\supp(\pi)$ to denote its support. Additionally, if $\pi$ corresponds to the joint distribution of two random variables $X$ and $Y$ (e.g. data and label), we use $\pi_X$ and $\pi_Y$ to denote the respective marginals, and $\pi_{X \mid Y = y}$ and $\pi_{Y \mid X = x}$ to denote the conditional distributions. For a random variable $X$ that has a density with respect to the Lebesgue measure, we use $p_X(x)$ to denote its density. For a general random variable $X$, we use $d\Prob_X$ to denote its associated probability measure, and use $d\Prob_X/d\Prob_Z$ to denote the Radon--Nikodym derivative when $X \ll Z$ (i.e. $X$ is absolutely continuous with respect to $Z$). We use $d_{\TV}$ to denote the total variation distance between probability measures. Lastly, we use $\Uni([a, b])$ to denote the uniform distribution on $[a, b]$ and $\Gauss(\mu, \sigma^2)$ to denote the Gaussian distribution with mean $\mu$ and variance $\sigma^2$.

We first consider calibration in the context of binary classification and then discuss generalizations to multi-class classifications. For a data distribution $\pi$ on $\R^d \times \{0, 1\}$, we say that a predictor $g: \R^d \to [0, 1]$ is calibrated if it satisfies the regular conditional probability condition $\E_{(X, Y) \sim \pi}[Y \mid g(X) = p] = p$. This condition corresponds to the idea that for all instances on which our model $g$ predicts probability $p$, the correct label of those instances is actually 1 with probability $p$.

The Expected Calibration Error (ECE) with respect to the distribution $\pi$ is then defined to be the expected absolute deviation from this condition:
\begin{align*}
    \ECE_{\pi}(g) \triangleq \E_{(X, Y) \sim \pi} \left[\abs{\E[Y \mid g(X)] - g(X)}\right]. \numberthis \label{ecedef}
\end{align*}

Although there are several ways to generalize ECE to the multi-class setting, perhaps the most reported generalization in the literature is top-class ECE. Namely, for a $k$-class classification problem in which we have a predictor $g: \R^d \to \Delta_k$, where $\Delta_k$ denotes the simplex of probability measures on a set with $k$ elements, the top-class ECE simply corresponds to computing the ECE with respect to the highest probability prediction:
\begin{align*}
    \E &\bigg[\bigg\vert\E[Y \in \argmax_i g^i(X) \big\vert \max_i g^i(X)] - \max_i g^i(X)\bigg\vert\bigg]. \numberthis \label{tcecedef}
\end{align*}

This definition is equivalent to considering the binary ECE with respect to a modified predictor $f$ and a distribution $(X', Y') \sim \pi'$ defined such that $Y' = \Ind_{Y \in \argmax_i g^i(X)}$ and $f(X') = \max_i g^i(X)$. As such, any modifications made to the binary version of ECE can be lifted to the multi-class setting via the top-class formulation of \eqref{tcecedef}, so we will henceforth just work with the binary version as defined in \eqref{ecedef}.

In practice, $\ECE_{\pi}(g)$ is estimated via binning. Given a set of data points $\{(x_1, y_1), ..., (x_n, y_n)\}$, one specifies a partition $B_1, B_2, ..., B_m$ of $[0, 1]$ and then computes
\begin{align*}
    \ECE_{\BIN, \pi}(g) \triangleq \sum_{j = 1}^m \frac{\abs{B_j}}{n} \abs{\bar{y}(B_j) - \bar{g}(B_j)}, \numberthis \label{binece}
\end{align*}
where $\bar{g}(B_j)$ corresponds to the average of all $g(x_i)$ such that $g(x_i) \in B_j$, and $\bar{y}(B_j)$ denotes the average over corresponding labels. In the multi-class case, one simply replaces $\bar{y}(B_j)$ with average accuracy and $\bar{g}(B_j)$ with average top-class probability.

\section{Continuity Properties of ECE}\label{sec:cont}
Having provided the necessary background regarding ECE, we now analyze its continuity properties. We begin first with the case where $\abs{\supp(\pi_X)} = n < \infty$, i.e. discrete data distributions with finite support. In this case, we can provide a necessary and sufficient condition for $\ECE_{\pi}$ to be discontinuous at $g$, which implies that $g$ can only be a point of discontinuity if it predicts the same probability at multiple points that each have positive measure under $\pi_X$. We subsequently show, however, that this intuition does not extend to the case where $\pi_X$ is supported on a more general (infinite) set. Nevertheless, with careful analysis, we can still extend the necessary and sufficient condition from the discrete case to a much more general setting.

\subsection{Discrete Distributions}\label{sec:discrete}
To get a sense of the continuity issues that arise with $\ECE_{\pi}$, we introduce an example adapted from \citep{blasiok2023unifying} that we will refer to multiple times throughout the rest of the paper.
\begin{restatable}{definition}{discretedist}[2-Point Distribution]\label{discretedist}
    Let $\pi$ be the distribution on $\{-1/2, 1/2\} \times \{0, 1\}$ such that $\pi_Y(0) = \pi_Y(1) = 1/2$ and $\pi_{X \mid Y = y}(x) = \Ind_{x = y - 1/2}$ (i.e. $X \mid Y = y$ is a point mass on $y - 1/2$).
\end{restatable}
It is straightforward to see that the predictor $g(x) = 1/2$ satisfies $\ECE_{\pi}(g) = 0$ for $\pi$ as in Definition \ref{discretedist}. However, perturbing $g$ such that $g(-1/2) = 1/2 - \varepsilon$ and $g(1/2) = 1/2 + \varepsilon$ yields $\ECE_{\pi}(g) = 1/2 - \varepsilon$ (for $\varepsilon \in (0, 1/2)$). The important idea here is that we split the level sets of $g$ by making an arbitrarily small perturbation, which causes the conditional expectation $\E[Y \mid g(X)]$ to jump from $1/2$ to $1$.\footnote{How the level sets of a predictor impact different loss functions applied to that predictor has also been studied more generally in the literature on scoring rules, in particular in the work of \citet{kull2015novel} which discusses the notion of a grouping loss.}

In fact, we can show that \textit{all} discontinuities of $\ECE_{\pi}$ for finitely supported $\pi$ occur at predictors that have non-singleton level sets with positive measure under $\pi_X$. This is part of the following full characterization of the discontinuities of $\ECE_{\pi}$ for discrete $\pi$.  (When we refer to discrete data distributions below, we always mean distributions with finite support.)
\begin{restatable}{theorem}{discretecase}[Discontinuities for Discrete ECE]\label{discretecase}
    Let $\pi$ be any distribution such that $\supp(\pi_X) = [n]$ for an arbitrary positive integer $n$, and let $g^*(x) = P(Y = 1 \mid X = x)$ denote the ground truth conditional distribution. Then the set of discontinuities of $\ECE_{\pi}$ (in the space of predictors $g: [n] \to [0, 1]$ endowed with the $\ell^\infty$ norm) is exactly the set of $g$ such that there exists $m \in [n]$ with $P(X = m) \ne 0$ and
    \begin{equation} \label{eqn:disc-cond}
    \big|g^*(m) - g(m) \big| \ne \big| \E[Y \, | \, g(X) = g(m)] - g(m) \big|.
    \end{equation}
\end{restatable}
Note that the choice of norm on the space of predictors makes no difference in this case, since when $\supp(\pi_X)$ is finite $g$ is just an $n$-dimensional vector, and all norms on $\R^n$ are equivalent. As promised, the proof of Theorem \ref{discretecase} relies on the following lemma, which shows that a discontinuity can only occur if $g$ predicts identical probabilities for at least two distinct points in $\supp(\pi_X)$.

\begin{restatable}{lemma}{disclemma}\label{lem:disc-nec}
Let $S(g, p) = \{j \in [n]: g(j) = p \textrm{ and } P(X = j) \ne 0\}$. If $P(X = m) \ne 0$ and \eqref{eqn:disc-cond} holds, then $|S(g,g(m))| > 1$.
\end{restatable}

\begin{proof}
Clearly $|S(g,g(m))| \ge 1$ since $m \in S(g,g(m))$, so it suffices to show that \eqref{eqn:disc-cond} fails if we assume $|S(g,g(m))| = 1$.  Under this assumption, we have
\begin{align*}
\E[Y \, | \, g(X) = g(m)] &= \frac{\sum_{j \in S(g, g(m))} \pi(j) g^*(j)}{\sum_{j \in S(g, g(m))} \pi(j)} \\
&= \frac{\pi(m) g^*(m)}{\pi(m)} = g^*(m), \numberthis \label{disclemmastep}
\end{align*}
so \eqref{eqn:disc-cond} indeed fails.
\end{proof}

A consequence of Lemma \ref{lem:disc-nec} is the following corollary, which shows that the set of discontinuities for $\ECE_{\pi}$ in the discrete case is negligible.
\begin{restatable}{corollary}{measurezero}\label{measurezero}
    If $\abs{\supp(\pi)} = n$, then the set of predictors $g$ (identified with vectors in $\R^n$) at which $\ECE_{\pi}$ is discontinuous has measure zero with respect to the Lebesgue measure on $\R^n$.
\end{restatable}
\begin{proof}
    By Lemma \ref{lem:disc-nec}, the set of discontinuities is a subset of the union of all sets $S_{i, j} = \{g: g(i) = g(j)\}$, where $i, j \in [n]$ and $i \neq j$. Since each $S_{i, j}$ has measure zero and we are considering a finite union of such sets, the set of discontinuities has measure zero.
\end{proof}

\begin{proof}[Proof of Theorem \ref{discretecase}]
    Write $p^* : [n] \to [0,1]$ for the probability mass function of $\pi_X$. Then we have:
    \begin{align*}
        \ECE_{\pi}(g) = \sum_{i = 1}^n p^*_i \abs{\frac{\sum_{j \in S(g, g(i))} p^*_j g^*(j)}{\sum_{j \in S(g, g(i))} p^*_j} - g(i)}. \numberthis \label{discece}
    \end{align*}
    Suppose now that there exists an $m \in [n]$ such that $P(X = m) \ne 0$ and \eqref{eqn:disc-cond} holds. Consider a new predictor $\tilde{g}$ such that $\tilde{g}(i) = g(i)$ for $i \ne m$, and $\tilde{g}(m) = g(m) + \delta$ for $\abs{\delta}$ small enough that $\tilde{g}(m) \in [0, 1]$ and $\abs{S(\tilde{g}, \tilde{g}(m))} = 1$. Then it follows that
    \begin{align*}
        \ECE_{\pi}(\tilde{g}) = \ECE_{\pi}(g) - p^*_m \abs{\frac{\sum_{j \in S(g, g(m))} p^*_j g^*(j)}{\sum_{j \in S(g, g(m))} p^*_j} - g(m)} + p^*_m \abs{g^*(m) - \tilde{g}(m)}, \numberthis \label{discperturb}
    \end{align*}
    which implies:
    \begin{align*}
        \abs{\ECE_{\pi}(g) - \ECE_{\pi}(\tilde{g})} &= p^*_m \abs{\abs{\frac{\sum_{j \in S(g, g(m))} p^*_j g^*(j)}{\sum_{j \in S(g, g(m))} p^*_j} - g(m)} - \abs{g^*(m) - \tilde{g}(m)}} \\
        &= p_m^* \bigg | \big| \E[Y \, | \, g(X) = g(m)] - g(m) \big| - |g^*(m) - \tilde{g}(m)| \bigg |.
        \numberthis \label{discerror}
    \end{align*}
    Thus we have $\lim_{\delta \to 0} \|g - \tilde{g}\|_\infty = 0$,
    whereas
    \[
    \lim_{\delta \to 0} \abs{\ECE_{\pi}(g) - \ECE_{\pi}(\tilde{g})} = p_m^* \bigg | \big| \E[Y \, | \, g(X) = g(m)] - g(m) \big| - |g^*(m) - g(m)| \bigg |,
    \]
    which is positive by \eqref{eqn:disc-cond}. Therefore $\ECE_{\pi}$ is discontinuous at $g$. 

    For the other direction, we show that if \eqref{eqn:disc-cond} does not hold for any $m \in [n]$ with $P(X = m) \ne 0$, then $\ECE_{\pi}$ is continuous at $g$. For any such $g$ and $\pi$, we have:
    \begin{align*}
        \ECE_{\pi}(g) = \sum_{i = 1}^n p^*_i \abs{g^*(i) - g(i)}. \numberthis \label{discece2}
    \end{align*}
    
    Now set $\delta = \min \abs{g(i) - g(j)}/2$, where the minimum runs over pairs of distinct $i \ne j \in [n]$ with $P(X = i), P(x = j) > 0$. By Lemma \ref{lem:disc-nec} we must have $|S(g,g(i))| = |S(g,g(j))| = 1$ for any such $i,j$, so that $\delta > 0$.  Then any $\tilde{g}$ satisfying $\norm{\tilde{g} - g}_{\infty} < \delta$ must also satisfy the property $\tilde{g}(i) \ne \tilde{g}(j)$ for any $i \ne j \in [n]$ with $P(X = i), P(x = j) > 0$. Therefore, again by Lemma \ref{lem:disc-nec}, \eqref{eqn:disc-cond} cannot hold with $\tilde{g}$ in the place of $g$ for any $m \in [n]$ with $P(X = m) \ne 0$, so that $\ECE_{\pi}(\tilde{g})$ has the same form as \ref{discece2}. We thus find
    \[
    \abs{\ECE_{\pi}(g) - \ECE_{\pi}(\tilde{g})} = \sum_{i = 1}^n p^*_i \bigg| \abs{g^*(i) - g(i)} - \abs{g^*(i) - \tilde{g}(i)} \bigg| \le \norm{\tilde{g} - g}_{\infty},
    \]
    which shows that $\ECE_{\pi}$ is continuous at $g$.
\end{proof}

\subsection{General Probability Measures}
Unfortunately, a key piece of intuition from the discrete setting fails to extend to the continuous case, as the following proposition shows.
\begin{restatable}{proposition}{counterexample}\label{counterexample}
    Take $\pi$ with $\supp(\pi_X) = [0, 1]^2$ such that $\pi_{X_1} = \Uni([0, 1])$,
    \begin{align*}
        \pi_{X_2 \mid X_1 = x_1} = x_1 \Uni([0.5, 1]) + (1 - x_1) \Uni([0, 0.5)),
    \end{align*}
    and $\Prob(Y = 1 \mid X_2 = x_2) = \Ind_{x_2 \geq 0.5}$. Then $\ECE_{\pi}$ is discontinuous at the predictor $g(x) = x_1$, despite the fact that $g$ has no level sets of positive measure.
\end{restatable}

\begin{proof}
    First, observe that the second assumption on $\pi$ implies
    \begin{align*}
        \Prob(Y = 1 \mid X_1 = x_1) &= \int_0^1 \Prob(Y = 1 \mid X_2 = x_2, X_1 = x_1) \Prob(X_2 = x_2 \mid X_1 = x_1) \ dx_2 \\
        &= \int_{0.5}^1 2x_1 \ dx_2 = x_1 ,\numberthis \label{x1cond}
    \end{align*}
    from which it follows that $\E[Y \mid g(X)] = g(X)$ and therefore $\ECE_{\pi}(g) = 0$. Now the idea is to apply an $L^{\infty}$ perturbation of size $\delta$ to $g$ such that we can separate the points for which $x_1$ is the same but $x_2 \ge 0.5$ or $x_2 < 0.5$, and in doing so obtain that the conditional probability (given $x_1$) of the label being 1 is either 0 or 1.

    Taking $\delta = 1/n$ for $n \in \mathbb{N}$ sufficiently large, we define the following two functions $g_{\delta, 0}$ and $g_{\delta, 1}$:
    \begin{align*}
        g_{\delta, 0}(z) &= \left(\bigg\lfloor \frac{z}{\delta} \bigg\rfloor + 1 - \frac{\delta}{4} \right) \delta, \numberthis \label{gdelta0} \\
        g_{\delta, 1}(z) &= \left(\bigg\lfloor \frac{z}{\delta} \bigg\rfloor + \frac{\delta}{4} \right) \delta. \numberthis \label{gdelta1}
    \end{align*}
    Clearly we have that $g_{\delta, 0}(z) \neq g_{\delta, 1}(z)$ for $z \in [0, 1]$ and that both $\abs{g_{\delta, 0}(z) - z} < \delta$ and $\abs{g_{\delta, 1}(z) - z} < \delta$. Additionally, $1 \ge g_{\delta, 0} > g_{\delta, 1}$ for $z \in [0, 1)$. Now we can define a perturbation of $g(x) = x_1$ by:
    \begin{align*}
        g_{\delta}(x) = \begin{cases}
            g_{\delta, 0}(x_1) & \text{if } x_2 < 0.5 \\
            g_{\delta, 1}(x_1) & \text{if } x_2 \ge 0.5.
        \end{cases} \numberthis \label{gdelta}
    \end{align*}
    We can then compute $\ECE_{\pi}(g_{\delta})$ as follows:
    \begin{align*}
        \ECE_{\pi}(g_{\delta}) &= \int \abs{\E[Y \mid g_{\delta}(x)] - g_{\delta}(x)} \ d\pi_X(x) \\
        &= \int \abs{\E[Y \mid g_{\delta}(x)] - g_{\delta}(x)} \Ind_{x_2 < 0.5} \ d\pi_X(x) + \int \abs{\E[Y \mid g_{\delta}(x)] - g_{\delta}(x)} \Ind_{x_2 \ge 0.5} \ d\pi_X(x) \\
        &\ge \int x_1 \Ind_{x_2 < 0.5} \ d\pi_X(x) + \int (1 - x_1) \Ind_{x_2 \ge 0.5} \ d\pi_X(x) - 2\delta \\
        &= \int_0^1 \int_0^{0.5} 2 x_1 (1 - x_1) \ dx_2 \ dx_1 + \int_0^1 \int_{0.5}^1 2 x_1 (1 - x_1) \ dx_2 \ dx_1 \\
        &= \frac{1}{3} - 2\delta. \numberthis \label{ecechange}
    \end{align*}
    Therefore $\ECE_{\pi}(g_{\delta}) \not \to 0$ as $\delta \to 0$, and thus $\ECE_{\pi}$ is discontinuous at $g$.
\end{proof}

In order to tackle the additional subtleties that arise when dealing with continuously distributed data, we take a completely general perspective. For the remainder of this section, we allow the data variable $X$ to take values in an arbitrary Polish space $\Omega_X$, endowed with its Borel $\sigma$-algebra $\mathcal{B}(\Omega_X)$.  Recall that a Polish space is, by definition, a separable completely metrizable topological space, so that this setting subsumes the case of finite discrete distributions treated above and also includes the case of continuously distributed vector-valued data (corresponding to $\Omega_X = \R^d$). In this general setting, we show that $\ECE_\pi$ is always lower semicontinuous on $L^p(\Omega_X; \pi_X)$ for $1 \le p \le \infty$, and we give a precise characterization of its points of discontinuity. 

We start with two lemmas.  The first is a straightforward and well-known result.
\begin{restatable}{lemma}{jensens} \label{lem:cnd-L1-ineq}
    Let $(\Omega, \Sigma, \mathbb{P})$ be a probability space, and let $\mathcal{F} \subseteq \mathcal{G} \subseteq \Sigma$ be sub-$\sigma$-algebras. Then for any $f \in L^1(\Omega;\mathbb{P})$, $\|\E[f|\mathcal{F}]\|_{1} \le \|\E[f|\mathcal{G}]\|_{1}$.
\end{restatable}

\begin{proof}
By the conditional Jensen's inequality and the tower property of conditional expectation,
\[
\big\|\E[f|\mathcal{F}]\big\|_{1} = \int_{\Omega} \big| \E[f|\mathcal{F}] \big| \, d \mathbb{P} = \int_{\Omega} \big| \E[ \, \E[f|\mathcal{G}]\, | \, \mathcal{F}] \big| \, d \mathbb{P} \le \int_{\Omega} \E \big[ \, | \E[f|\mathcal{G}]| \, \big| \, \mathcal{F} \big] \, d \mathbb{P} = \big\|\E[f|\mathcal{G}]\big\|_{1}.\qedhere
\]
\end{proof}

The second lemma is, to our knowledge, new. The proof relies on a construction due to \citet{KudoConv}.

\begin{restatable}{lemma}{sigmaliminf} \label{lem:sigma-liminf}
Let $(\Omega, \Sigma, \mathbb{P})$ be a probability space, let $f \in L^1(\Omega;\mathbb{P})$, and let $g, g_1, g_2, \hdots$ be real-valued random variables such that $g_n \to g$ in probability. Then $\liminf_{n\to \iy} || \E[f | g_n] ||_1 \ge ||\E[f|g]||_1$.
\end{restatable}

\begin{proof}
We write $\sigma(g)$ for the $\sigma$-algebra generated by preimages of Borel sets under $g$.  The {\it lower limit} of the sequence of $\sigma$-algebras $\sigma(g_1), \sigma(g_2), \hdots$, defined by  \cite{KudoConv}, is the $\sigma$-algebra
\[
\mathcal{G} = \Big \{ A \in \Sigma \ \Big | \ \lim_{n\to\infty} \inf_{B \in \sigma(g_n)} \mathbb{P}(A \Delta B) = 0 \Big \},
\]
where $\Delta$ indicates the symmetric difference of sets.  The lower limit satisfies the property that
\begin{equation} \label{eqn:sigma-ll}
\liminf_{n \to \infty} \int_\Omega \big| \E[h|g_n] \big| \, d\mathbb{P} \ge \int_\Omega \big| \E[h|\mathcal{G}] \big| \, d\mathbb{P}
\end{equation}
for any bounded $\Sigma$-measurable function $h$, and if $\mathcal{F}$ is any $\sigma$-algebra such that (\ref{eqn:sigma-ll}) holds with $\mathcal{F}$ in the place of $\mathcal{G}$, then $\mathcal{F} \subset \mathcal{G}$.

By \eqref{eqn:sigma-ll} and Lemma \ref{lem:cnd-L1-ineq}, to prove the desired result it suffices to show that $\sigma(g) \subset \mathcal{G}$, and thus it suffices to show that some generating set of $\sigma(g)$ is contained in $\mathcal{G}$.  Let $C$ be the set of atoms of the pushforward distribution $g_* \mathbb{P}$ of $g$ on $\R$.  Then $C$ is at most countable, and thus sets of the form
\begin{align}\label{e:Ax}
A = \{ \omega \in \Omega \ | \ g(\omega) < x \},
\end{align}
for $x \in \R \setminus C$, generate $\sigma(g)$.  We will show that $\lim_{n\to\infty} \inf_{B \in \sigma(g_n)} \mathbb{P}(A \Delta B) = 0$, so that $A \in \mathcal{G}$.

Fix a set $A$ of the form \eqref{e:Ax} and $\varepsilon > 0$, and let $B_{n,\varepsilon} = \{ \omega \in \Omega \ | \ g_n(\omega) < x + \varepsilon \}$. We then have
\[
\mathbb{P}(|g_n - g| > \varepsilon) \ge \mathbb{P}(A \Delta B_{n,\varepsilon}) - \mathbb{P}(x \le g \le x + 2\varepsilon),
\]
so that
\[
\inf_{B \in \sigma(g_n)} \mathbb{P}(A \Delta B) \le \mathbb{P}(|g_n - g| > \varepsilon) + \mathbb{P}(x \le g \le x + 2\varepsilon),
\]
and since $g_n \to g$ in probability, we obtain
\[
\lim_{n \to \infty} \inf_{B \in \sigma(g_n)} \mathbb{P}(A \Delta B) \le \mathbb{P}(x \le g \le x + 2\varepsilon).
\]
Since $x$ is not an atom of $g_*\mathbb{P}$, the right-hand side above can be made arbitrarily small.  Therefore $\lim_{n\to\infty} \inf_{B \in \sigma(g_n)} \mathbb{P}(A \Delta B) = 0$, which completes the proof.
\end{proof}

We now can prove the first main result of this section.

\begin{theorem} \label{thm:lower-semicont}
    The functional $\ECE_\pi$ is lower semicontinuous on $L^p(\Omega_X; \pi_X)$ for $1 \le p \le \infty$.
\end{theorem}

\begin{proof}
    Let $g, g_1, g_2, \hdots \in L^p(\Omega_X; \pi_X)$ and suppose $g_n \to g$.  We will show that $\liminf_{n\to \iy} \ECE_\pi(g_n) \ge \ECE_\pi(g)$.

    Since $g_n \to g$ in $L^p(\Omega_X; \pi_X)$, we have the convergence $g_n(X) \to g(X)$ of $L^p$ random variables on the background probability space, and in particular, $g_n(X) \to g(X)$ in probability.  Then, by Lemma \ref{lem:sigma-liminf} and $L^p$-continuity of conditional expectation, we have
    \begin{align*}
    \liminf_{n \to \infty} \ECE_\pi(g_n) &= \liminf_{n \to \infty} \big\| \E[Y-g_n(X) | g_n(X)] \big\|_1 \\
    &= \liminf_{n \to \infty} \big\| \E[Y-g(X) | g_n(X)] + \E[g(X)-g_n(X) | g_n(X)]\big\|_1 \\&= \liminf_{n \to \infty} \big\| \E[Y-g(X) | g_n(X)] \big\|_1 \\
    &\ge \big\| \E[Y-g(X) | g(X)] \big\|_1 = \ECE_\pi(g),
    \end{align*}
    which is the desired result.
\end{proof}

Below we use Theorem \ref{thm:lower-semicont} to prove the necessity direction of our general condition for $\ECE_\pi$ to be continuous at a point in $L^p$.  To prove that this same condition is also sufficient, we will need the following lemma.

Recall that a standard Borel space is a Polish space equipped with its Borel $\sigma$-algebra.  A measurable bijection between standard Borel spaces is always an isomorphism (i.e., its inverse is also measurable).  The Kuratowski isomorphism theorem states that two standard Borel spaces are isomorphic if and only if they have the same cardinality; in particular, every standard Borel space is isomorphic to one of $\mathbb{R}$, $\mathbb{Z}$, or a finite discrete space.

\begin{restatable}{lemma}{injdense} \label{lem:inj-dense}
    Let $(\Omega, \mathcal{B}(\Omega), \mathbb{P})$ be a probability space, where $\Omega$ is a Polish space and $\mathcal{B}(\Omega)$ is its Borel $\sigma$-algebra.  For $1 \le p \le \infty$ define
    \begin{align*}
    L^p_{\mathrm{inj}}(\Omega ; \mathbb{P}) = \big\{ \, f \in L^p(\Omega ; \mathbb{P}) \ \big |
     \ f &\textrm{ is almost surely equal to a bijection onto a standard Borel space }\big\}.
    \end{align*}
    Then $L^p_{\mathrm{inj}}(\Omega ; \mathbb{P})$ is dense in $L^p(\Omega ; \mathbb{P})$.
\end{restatable}

\begin{proof}[Proof of Lemma~\ref{lem:inj-dense}]
Although the lemma holds for the full $L^p$ space of complex-valued functions, it is sufficient to show the result for the subspace of real-valued functions in $L^p$, and we will only need to use this case below.  Moreover, by the Kuratowski isomorphism theorem, it is sufficient to consider the cases $\Omega = \mathbb{R}$, $\mathbb{Z}$, or $[n]$.  We take $\Omega = \mathbb{R}$; the cases $\Omega = \mathbb{Z}$ or $[n]$ can be treated by a simpler version of the same argument.

Given a real-valued $f \in L^p(\Omega ; \mathbb{P})$, we can choose simple functions
\[
f_n = \sum_{j=1}^n a^{(n)}_j \mathbbm{1}_{A^{(n)}_j},
\]
where $A^{(n)}_1, \hdots, A^{(n)}_n \subset \Omega$ are pairwise disjoint open or half-open intervals and $a^{(n)}_n \ge \hdots \ge a^{(n)}_1 \in \R$, such that $f_n \to f$ in $L^p$.  Again by the Kuratowski isomorphism theorem, there exist measurable bijections $\varphi^{(n)}_j : A^{(n)}_j \to (0,1)$.  For $\varepsilon > 0$, define
\[
f_{n,\varepsilon} = \sum_{j=1}^n \big[ a^{(n)}_j + \varepsilon(\varphi^{(n)}_j + 2j - 1)\big]\mathbbm{1}_{A^{(n)}_j}.
\]
Observe that $f_{n,\varepsilon}$ maps each set $A^{(n)}_j$ bijectively to the interval $(a^{(n)}_j + (2j-1)\varepsilon, \, a^{(n)}_j + 2j\varepsilon)$, and
for $\varepsilon$ sufficiently small,  
these latter intervals have pairwise disjoint closures for different values of $j$. Therefore, for all small enough $\varepsilon_n>0$,  $f_{n,\varepsilon_n}$ is injective, and thus a measurable bijection onto its image. This image is a union of finitely many open intervals with pairwise disjoint closures, which is a Polish space, and thus $f_{n,\varepsilon_n}$ is an isomorphism from $\Omega$ to a standard Borel space.
Moreover $\|f_n - f_{n,\varepsilon}\|_{p} \le 2 n \varepsilon$, so that additionally taking $\varepsilon_n$ small enough that $n\varepsilon_n\to 0$, we have that $f_{n,\varepsilon_n}$ is a sequence in $L^p_{\mathrm{inj}}(\Omega ; \mathbb{P})$ that converges to $f$ in $L^p$.
\end{proof}

Putting everything together, we prove the second main result of this section, which gives a full characterization of the points of continuity of $\ECE_{\pi}$.
\begin{theorem}
\label{thm:ece-continuity}
    Let $1 \le p \le \infty$, let $g \in L^p(\Omega_X ; \pi_X)$, and set $g^*(x) = \E[Y \, | \, X = x]$ for $x \in \Omega_X.$ Then $\ECE_\pi$ is continuous at $g$ in the topology of $L^p$ if and only if
    \begin{equation} \label{eqn:cty-criterion}
        \ECE_\pi(g) = \| g^* - g \|_{L^1(\Omega_X ; \pi_X)}.
    \end{equation}
\end{theorem}

\begin{proof}
We first show that if (\ref{eqn:cty-criterion}) holds, then $\ECE_\pi$ is continuous at $g$.  Suppose that $\ECE_\pi(g) = \| g^* - g \|_{L^1(\Omega_X ; \pi_X)}$, and let $g_1, g_2, \hdots \in L^p(\Omega_X; \pi_X)$ be any sequence converging to $g$.  By Theorem \ref{thm:lower-semicont}, $\liminf_{n\to \iy} \ECE_\pi(g_n) \ge \ECE_\pi(g)$.  Since $\sigma(g_n(X)) \subset \sigma(X)$, by Lemma \ref{lem:cnd-L1-ineq} we have
\begin{align*}
\limsup_{n \to \infty} \ECE_\pi(g_n) &= \limsup_{n \to \infty} \big\| \E[Y-g_n | g_n(X)] \big \|_1 \\
&\le \limsup_{n \to \infty} \big\| \E[Y-g_n | X] \big \|_1 \\ 
&= \limsup_{n \to \infty} \big\| g^* - g_n \big\|_{L^1(\Omega_X ; \pi_X)} \\
&= \big\| g^* - g \big\|_{L^1(\Omega_X ; \pi_X)} = \ECE_\pi(g).
\end{align*}
Therefore $\ECE_\pi(g_n) \to \ECE_\pi(g)$, which shows that $\ECE_\pi$ is continuous at $g$.

For the opposite direction of implication, we show that if (\ref{eqn:cty-criterion}) does not hold, then $\ECE_\pi$ is discontinuous at $g$. Suppose that $\ECE_\pi(g) \ne \| g^* - g \|_{L^1(\Omega_X ; \pi_X)}$.  By Lemma \ref{lem:inj-dense}, we can choose a sequence $g_n \to g$ in $L^p(\Omega_X; \pi_X)$ such that each $g_n$ is a bijection onto a standard Borel space. Since a measurable bijection between standard Borel spaces has a measurable inverse, this implies that $g_n^{-1}$ is a measurable bijection from the image of $g_n$ onto $\Omega_X$, which in turn implies that $\si(g_n(X)) = \si(X)$. Thus
\[
\E[Y \, | \, g_n(X) = g_n(x)] = \E[Y \, | \, X=x] = g^*(x),
\]
so that
\[
\ECE_\pi(g_n) = \| g^* - g_n \|_{L^1(\Omega_X ; \pi_X)}.
\]
The right-hand side above is a continuous functional of $g_n$ in $L^p(\Omega_X; \pi_X)$, so that
\[
\lim_{n \to \infty} \ECE_\pi(g_n) = \| g^* - g \|_{L^1(\Omega_X ; \pi_X)} \ne \ECE_\pi(g),
\]
implying that $\ECE_\pi$ is discontinuous at $g$.
\end{proof}

We leave it as an exercise to verify that if $\Omega_X = [n]$ with the discrete topology, then the condition in Theorem \ref{thm:ece-continuity} is equivalent to the condition in Theorem \ref{discretecase}.

\section{Logit Smoothed Calibration}\label{sec:lsece}
We proved in Corollary \ref{measurezero} above that for finitely supported distributions, the discontinuities of ECE have measure zero.  Although this statement only holds as written for discrete data, it nonetheless provides helpful intuition that we can use to mitigate the discontinuities of ECE in a more general setting: namely, we expect that predictors at which ECE is discontinuous should be, in some sense, rare.  Therefore, if we add some independent continuously distributed random noise $\xi$ to $g(X)$ before taking the conditional expectation in \eqref{ecedef}, we can hope that the resulting functional of $g$ will be continuous.  We show below that indeed this is the case.

However, in order to preserve the interpretation of $\ECE_{\pi}$, we need to ensure that $g(X) + \xi \in [0, 1]$. To do so, we assume that $g$ can be decomposed as $\rho \circ h$ where $h: \R^d \to \R$ and $\rho: \R \to [0, 1]$ is a strictly increasing function (e.g., $\rho$ can be the sigmoid function), observing that this is virtually always the case in practice.\footnote{We only run into issues if $g$ takes the values 0 or 1 exactly, since then $\rho^{-1}$ can be $-\infty$ and $\infty$ respectively. This is easily avoided in practice by adding/subtracting a small tolerance value to the predicted probabilities, if necessary.} We can then add the noise $\xi$ to $h(X)$ rather than $g(X)$, which allows us to define the \textit{Logit-Smoothed ECE (LS-ECE)} as follows:
\begin{align*}
    \SECE_{\pi, \xi}(h) \triangleq \E_{X, \xi} &\big[\big\vert\E[Y \mid \rho(h(X) + \xi)] - \rho(h(X) + \xi)\big\vert\big]. \numberthis \label{secedef}
\end{align*}

Henceforth, we will always use $h(X)$ to denote the \textit{logit} associated with $g(X)$, and $\rho: \R \to [0, 1]$ to denote a strictly increasing function with inverse $\rho^{-1}$ differentiable everywhere on $(0, 1)$.

\textbf{Comparison to \citet{blasiok2023smooth}.} While the main proposal of $\smECE$ in \citet{blasiok2023smooth} does not exactly match our notion of $\SECE$ (as it corresponds to smoothing residuals of the predictions), we note that their notion of $\smECEp$ shares more similarity to what we propose. Namely, $\smECEp$ corresponds to smoothing the predictor $g$ and then projecting back to $[0, 1]$. However, smoothing the logit function $h$ directly avoids issues related to thresholding/projecting and leads to a cleaner development of the theory in this section, allowing us to prove continuity, consistency, and convergence to the true ECE under reasonable assumptions.

\subsection{Continuity of LS-ECE}
We now verify that unlike $\ECE_{\pi}$, $\SECE_{\pi, \xi}(h)$ is continuous as a function of the logit $h$ in the topology of $L^{\infty}$ for \textit{any} choice of $\pi$ so long as $\xi$ has a density with respect to the Lebesgue measure and satisfies very basic regularity conditions. The crux of our argument relies on analyzing how perturbations to the joint distribution of $(Y, \rho(h(X) + \xi))$ behave with respect to total variation, and then ``pulling back'' to continuity in the space of logit functions $h$.

We begin by showing in the next two propositions that the smoothed logits $h(X) + \xi$ are continuous in total variation with respect to the $L^p$ norm on $h$.
\begin{restatable}{proposition}{tvprop}\label{tvprop}
    Let $Z_n$ denote a sequence of random variables converging to a random variable $Z$ in $L^p$ for $p \in [1, \infty]$, and let $\xi$ be an independent, real-valued random variable with density $p_{\xi}$ that is continuous Lebesgue almost everywhere.
    Suppose $Z_n, Z$ are $X$-measurable for a random variable $X$. 
    Then $(X,Z_n + \xi) \to (X, Z + \xi)$ in total variation.
\end{restatable}
\begin{proof}
    By assumption, $p_\xi(x-\ep) \to p_\xi(x)$ as $\ep\to 0$, almost everywhere. By Scheff\'e's Lemma, $d_\TV(\xi, \xi+\ep) \to 0$ as $\ep\to 0$. Hence, for all $\ep>0$, there exists $\de>0$ such that $d_{\TV}(\xi, \xi+\de')<\ep$ for all $|\de'|<\de$. 

    Note we may assume $p=1$. 
    Choosing $\de$ based on $\ep$, we have
\begin{align*}
    d_\TV((X,Z_n+\xi),(X,Z+\xi)) &\le 
    \iint d_\TV(z_n+\xi, z+\xi) \,d\Prob_{Z_n,Z|X=x}(z_n,z)\,d\Prob_X(x)\\
    &\le \iint d_\TV(\xi, z-z_n+\xi)\,d\Prob_{Z_n,Z|X=x}(z_n,z)\,d\Prob_X(x)\\
    &\le \iint_{|z_n-z|<\de} d_\TV(\xi, z-z_n+\xi)\,d\Prob_{Z_n, Z|X=x}(z_n,z)\,d\Prob_X (x) + \Prob(|Z_n-Z|\ge \de)\\
    &\le \ep + \Prob(|Z_n-Z|\ge \de)\\
    &\le \ep + \fc{\E|Z_n-Z|}{\de} \numberthis \label{tvpropstep}
\end{align*}
where the last step follows from Markov's inequality.
Choose $N$ such that for $n\ge N$, $\E|Z_n-Z|\le \de \ep$. Then for $n\ge 2N$, we have $d_\TV((X,Z_n+\xi),(X,Z+\xi))\le 2\ep$. 
\end{proof}

\begin{remark}
    An example of a density that is not continuous Lebesgue almost everywhere is an indicator on a fat Cantor set, which is discontinuous on a set of positive Lebesgue measure. By the Lebesgue differentiation theorem, any equivalent density must agree almost everywhere, and hence still be discontinuous on a set of positive measure.
\end{remark}

\begin{restatable}{proposition}{tvlemma}\label{tvlemma}
    Let $\xi$ be as in Proposition \ref{tvprop} and let $(X, Y) \sim \pi$. 
    Suppose that $h_n(X) \to h(X)$ in $L^p$ for some $p \in [1, \infty]$.
    Then $(Y, \rho(h_n(X) + \xi)) \tvto (Y, \rho(h(X) + \xi))$.
\end{restatable}
\begin{proof}
    Let $Z_n = h_n(X)$ and $Z = h(X)$. 
    By Proposition \ref{tvprop}, $(X, Z_n+\xi)\tvto (X,Z+\xi)$. 
    By the data processing inequality, applying the kernel $\pi_{Y|X}$ to the first argument and $\rh$ to the second argument, $(Y,\rh(Z_n+\xi))\tvto(Y,\rh(Z+\xi))$.
\end{proof}

The final ingredient necessary for our proof of continuity is the connection between convergence in total variation of joint distributions and convergence of the associated ECEs. We provide this via the following general result, which may be of independent interest for future analysis of ECE.
\begin{restatable}{lemma}{tvdelta}\label{l:tv-delta}
    Suppose that $(Y,T_n)\to (Y,T)$ in total variation, where $T_n$, $T$ are random variables taking values in $[0,1]$. 
    Define \begin{align*}
        \Delta_n = \big\vert&\E_{T_n} \left[\abs{\E[Y \mid T_n = t] - t}\right] - \E_{T} \left[\abs{\E[Y \mid T = t] - t}\right]\big\vert. \numberthis \label{secediff}
    \end{align*}
    Then $\lim_{n\to \iy} \De_n=0$.
\end{restatable}
\begin{proof}[Proof of Lemma~\ref{l:tv-delta}]
    Let $\varepsilon_1 > 0$ be arbitrary. 
    As a first step, we will apply a change of measure to write both expectations in \eqref{secediff} in terms of a single random variable. 
Let $S_n$ denote a random variable that is distributed as $T_n$ with probability $\rc2$ and $T$ with probability $\rc2$. Note that $\Prob_{S_n}$ is mutually absolutely continuous with both $\Prob_{T_n}$ and $\Prob_T$. 
Then by Markov's inequality, for $U\in \{T_n,T\}$,
\begin{align*}
    \Prob_{S_n} \pa{
        \ab{\dr{\Prob_{U}}{\Prob_{S_n}} -1}\ge \ep_1
    }\le 
    \fc{\E_{S_n}\ab{\dr{\Prob_{U}}{\Prob_{S_n}} - 1}}{\ep_1}
    =
    \fc{d_\TV(T_n,T)}{\ep_1}.
\end{align*}
By Lemma \ref{tvlemma}, for any $\de_1>0$, we can choose $n$ sufficiently large so that $d_\TV(T_n,T)\le \ep_1\de_1$, so that 
    \begin{align}
        \Prob_{S_n}\left(\ab{\frac{d\Prob_{X}}{d\Prob_{S_n}} - 1} \ge \ep_1\right) \le \delta_1.
        \label{e:dd-markov}
    \end{align}
    We change the measure using $\Prob_{T_n}, \Prob_T \ll \Prob_{S_n}$:
    \begin{align}\label{e:Den1}
        \Delta_n 
        &\le 
        \ab{\E_{S_n}\ba{\dr{\Prob_{T_n}}{\Prob_{S_n}}\ab{\E[Y\mid T_n=t]-t}}}
        - \ab{\E_{S_n}\ba{\dr{\Prob_T}{\Prob_{S_n}}\ab{\E[Y\mid T=t]-t}}}.
    \end{align}
    We compare each term to $\E_{S_n}\ba{\ab{\E[Y|U=t]-t}}$ for $U=T_n, T$, respectively. By the triangle inequality,
    \begin{align*}
        \ab{\E_{S_n}\ba{\dr{\Prob_{U}}{\Prob_{S_n}}\ab{\E[Y\mid U=t]-t}}}
        - 
        \ab{\E_{S_n}\ba{\ab{\E[Y\mid U=t]-t}}}
        &\le 
        \ab{\E_{S_n}\ba{\ab{\dr{\Prob_{U}}{\Prob_{S_n}}-1}\ab{\E[Y\mid U=t]-t}}}\\
        &\le 
        \Prob\pa{\ab{\dr{\Prob_U}{\Prob_{S_n}}-1}\ge \ep_1} + \ep_1\le \de_1+\ep_1
        \numberthis\label{e:Den2}
    \end{align*}
    where we split the expectation depending on whether $\ab{\dr{\Prob_{T_n}}{\Prob_{S_n}}- \dr{\Prob_T}{\Prob_{S_n}}}\ge \ep_1$, note that 
    $\dr{\Prob_U}{\Prob_{S_n}}\le 2$ and 
    $\ab{\E[Y|U=t]-t}\le 1$, and use Lemma~\ref{e:dd-markov}.
    From~\eqref{e:Den1} and~\eqref{e:Den2}, 
    \begin{align*}
        \Delta_n 
        &\le 
        2(\ep_1+\de_1) + 
        \E_{S_n}\ba{\ab{\E[Y\mid T_n=t]-t}}
        - 
        \E_{S_n}\ba{\ab{\E[Y\mid T=t]-t}}\\
        &\le 2(\ep_1+\de_1)
        +
        \E_{S_n}
        \ba{\ab{\E[Y\mid T_n=t]-\E[Y\mid T=t]}},
        \numberthis\label{e:Den}
    \end{align*}
    where we used the triangle inequality. 
    To simplify notation moving forward, we will use $d\Prob_{1, U}$ to denote the density of $(Y = 1, U)$. 
    Finally,
    \begin{align*}
        \E_{S_n}
        \ba{\ab{\E[Y\mid T_n=t]-\E[Y\mid T=t]}}
        &=
        \E_{S_n} \ba{\ab{
        \dr{\Prob_{1,T_n}}{\Prob_{T_n}} - \dr{\Prob_{1,T}}{\Prob_{T}}}}\\
        &= 
        \E_{S_n} 
        \ba{\ab{\dr{\Prob_{1,T_n}}{\Prob_{S_n}} \dr{\Prob_{S_n}}{\Prob_{T_n}} - 
        \dr{\Prob_{1,T}}{\Prob_{S_n}} \dr{\Prob_{S_n}}{\Prob_{T}}
        }}
        \numberthis
        \label{e:EY}
    \end{align*}
    For $U\in \{T_n, T\}$, we know $\dr{\Prob_{S_n}}{\Prob_U} \in [1-O(\ep_1), 1+O(\ep_1)]$ with probability $\ge 1-\de_1$ by \eqref{e:dd-markov}. For the first factors, we have $\dr{\Prob_{1,U}}{\Prob_{S_n}} \approx \dr{\Prob_{1,S_n}}{\Prob_{S_n}}$ from the same logic as \eqref{e:dd-markov}:
    \[
\Prob_{S_n}
\pa{
    \ab{\dr{\Prob_{1,U}}{\Prob_{S_n}} - \dr{\Prob_{1,S_n}}{\Prob_{S_n}}}\ge \ep_1
}
\le \fc{\E_{S_n} \ba{\ab{\dr{\Prob_{1,U}}{\Prob_{S_n}} - \dr{\Prob_{1,S_n}}{\Prob_{S_n}}}}}{\ep_1}
= \fc{d_\TV((Y,T_n), (Y,T))}{\ep_1}\le \de_1
    \]
    where we now use $(Y,T_n)\tvto (Y,T)$ and choose $n$ sufficiently large so that $d_\TV((Y,T_n), (Y,T))\le \ep_1\de_1$. By comparing both terms in~\eqref{e:EY} to $\dr{\Prob_{1,S_n}}{\Prob_{S_n}}$, we then get that \eqref{e:EY} is $O(\ep_1+\de_1)$. Together with~\eqref{e:Den}, we have $\De_n = O(\de_1+\ep_1)$, and since $\de_1,\ep_1$ were arbitrary we have that $\lim_{n\to \iy} \Delta_n=0$, as desired.  
\end{proof}

We can now prove the main result of this section.
\begin{restatable}{theorem}{sececontthm}\label{sececontthm}
    Let $\xi$ satisfy the conditions of Proposition \ref{tvprop}. Then $\SECE_{\pi, \xi}(h)$ is a continuous functional of $h$ in the topology of $L^{\infty}$.
\end{restatable}

\begin{proof}[Proof of Theorem~\ref{sececontthm}]
    Let $h_n$ denote a sequence of functions converging to $h$ in $L^{\infty}$. Let $Z_n = h_n(X)$, $Z = h(X)$ and $T_n = \rho(Z_n + \xi)$, $T = \rho(Z + \xi)$. By Lemma \ref{tvlemma}, we have that $(Y, T_n) \to (Y, T)$ in total variation. By Lemma~\ref{l:tv-delta}, 
    $\ab{\SECE_{\pi, \xi}(h_n) -\SECE_{\pi, \xi}(h)} = \De_n \to 0$
    as $n\to \iy$.
\end{proof}

\section{Estimation of LS-ECE}\label{sec:est}
Having established the continuity of $\SECE_{\pi, \xi}$, we turn to its estimation in practice. Let $\{(x_1, y_1), ..., (x_n, y_n)\}$ denote $n$ points sampled from the data distribution $\pi$, and let $\hat{\pi}$ denote the empirical measure of the pairs $(x_i, y_i)$. Then we can naturally approximate $\SECE_{\pi, \xi}$ by $\SECE_{\hat{\pi}, \xi}$, and then estimate $\SECE_{\hat{\pi}, \xi}$ by estimating the outer expectation in the definition of $\SECE_{\hat{\pi}, \xi}$ via sampling.

We first explicitly derive the form of $\E[Y \mid \rho(h(X) + \xi)]$ in the population case of $(X, Y) \sim \pi$, as this will make clear the form of $\SECE_{\hat{\pi}, \xi}$. This requires the following elementary proposition.
\begin{restatable}{proposition}{convprop}\label{convprop}
    Let $Z$ be an arbitrary real-valued random variable and let $\xi$ be a real-valued random variable with density $p_\xi$. Then $Z + \xi$ has the following density with respect to the Lebesgue measure:
    \begin{align*} 
        p_{Z+\xi}(t)
        = \E_Z [p_{\xi}(t - Z)]. \numberthis \label{convdensity}
    \end{align*}
\end{restatable}
\begin{proof}
    We have that:
    \begin{align*}
        \Prob(Z +\xi \le t) 
        &= \int_{-\infty}^{\infty} \Prob (\xi \le t-z)\ d\Prob_Z(z) \\
        &= \int_{-\infty}^{\infty} \int_{-\infty}^{t - z} p_{\xi}(u) \ du \ d\Prob_Z(z) \\
        &= \int_{-\infty}^{t} \int_{-\infty}^{\infty} p_{\xi}(u - z) \ d\Prob_Z(z) \ du \numberthis \label{convproof}
    \end{align*}
    from Fubini's theorem and translation invariance of the Lebesgue measure. 
\end{proof}

For brevity we now let $T = \rho(h(X) + \xi)$ with $(X, Y) \sim \pi$. Then we have from Proposition \ref{convprop} that $\E[Y \mid T = t] = p_{T, Y = 1}(t)/p_T(t)$, where the densities $p_{T, Y=1}$ and $p_T$ are:
\begin{align*}
    p_{T, Y=1}(t) &= \pi_Y(1) (\rho^{-1})'(t) \E\big[p_{\xi}(\rho^{-1}(t)  - h(X)) \mid Y = 1\big], \numberthis \label{ptdensity} \\
    p_T(t) &= (\rho^{-1})'(t) \E\left[p_{\xi}(\rho^{-1}(t) - h(X))\right]. \numberthis \label{ptdensity2}
\end{align*}
Now if we let $\hat{T}$ be analogous to $T$ except with $(X, Y) \sim \hat{\pi}$, then we can similarly obtain the expression $\E[Y \mid \hat{T} = t] = p_{\hat{T}, Y = 1}(t)/p_{\hat{T}}(t)$, with the densities:
\begin{align*}
    p_{\hat{T}, Y = 1}(t) &= \frac{1}{n}\sum_{i = 1}^n (\rho^{-1})'(t) p_{\xi}(\rho^{-1}(t) - h(x_i)) \Ind_{y_i = 1}, \numberthis \label{tndensity2} \\
    p_{\hat{T}}(t) &= \frac{1}{n} \sum_{i = 1}^n (\rho^{-1})'(t) p_{\xi}(\rho^{-1}(t) - h(x_i)). \numberthis \label{tndensity}
\end{align*}
From here, it is straightforward to estimate $\SECE_{\hat{\pi}, \xi}(h)$ by averaging over samples from $\hat{T}$, so long as we take $p_{\xi}$ to be easy to sample from. We also see that $\E[Y \mid \hat{T} = t]$ is just the Nadaraya--Watson kernel regression estimator \citep{nadaraya1964regression, watson1964regression} evaluated at $t$ using $h$ and a kernel corresponding to the density of $\xi$, as we would expect. The form of $\E[Y \mid \hat{T} = t]$ also makes it trivial to implement estimation of $\SECE_{\hat{\pi}, \xi}(h)$ in practice, as we illustrate in Figure \ref{logitcode}.

\begin{figure}[h]
\centering
\includegraphics[width=0.95\columnwidth]{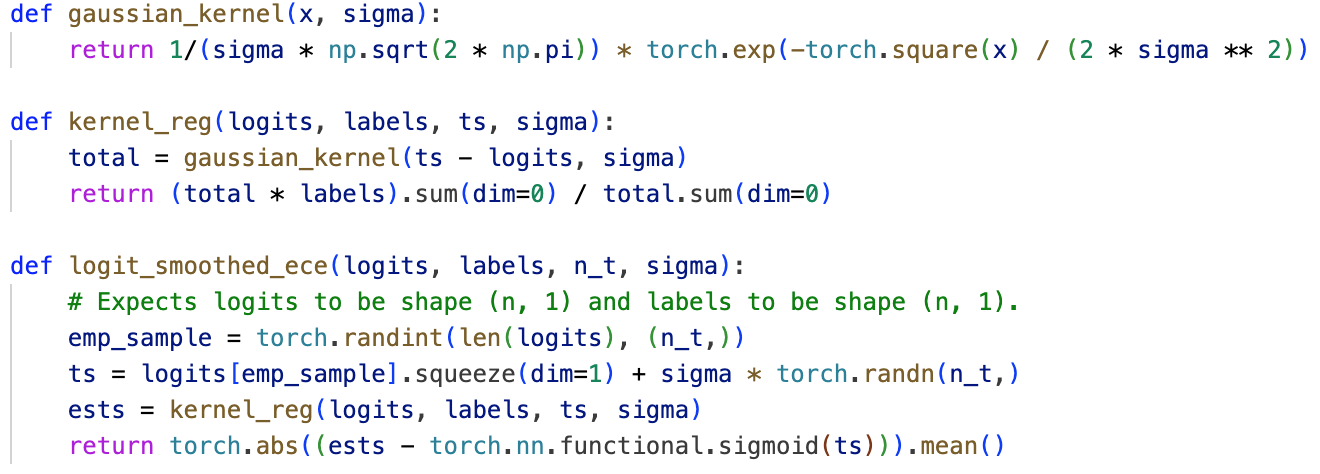} 
\caption{Implementation of $\SECE_{\hat{\pi}, \xi}(h)$ in 10 lines of PyTorch \citep{pytorch} using broadcast semantics.}
\label{logitcode}
\end{figure}

We need only now prove that $\SECE_{\hat{\pi}, \xi}(h)$ converges in probability to $\SECE_{\pi,\xi}(h)$, i.e. that estimating $\SECE_{\hat{\pi}, \xi}$ allows us to consistently estimate $\SECE_{\pi,\xi}(h)$. We note that here $\SECE_{\pi,\xi}(h)$ is a non-random scalar quantity depending on the logit function $h$ while $\SECE_{\hat{\pi}, \xi}(h)$ is a random variable depending on the pairs $(x_i, y_i) \sim \pi$. 

It turns out that the only additional stipulation on $\xi$ necessary to achieve consistency is that $\xi$ be of the form $\xi = \sigma R$ for a random variable $R$ with bounded density (for example, $\xi \sim \Gauss(0, \sigma^2)$). So long as $\xi$ is of this type, we can prove that $p_{\hat{T}} \to p_T$ and $p_{\hat{T}, Y = 1} \to p_{T, Y = 1}$ in $L^1$, which is all we need for consistency. This is encapsulated in the following lemma, which once again may be of independent interest.
\begin{restatable}{lemma}{varlemma}\label{varlemma}
    If $\xi = \sigma R$ for a random variable $R$ with bounded density, then:
    \begin{align*}
        \E_{\pi}\left[\int_0^1 \abs{p_{T, Y = 1}(t) - p_{\hat{T}, Y = 1}(t)} \ dt\right] &= O\left(\frac{1}{\sqrt{n\sigma}}\right), \numberthis \label{condptbound} \\
        \E_{\pi}\left[\int_0^1 \abs{p_{T}(t) - p_{\hat{T}}(t)} \ dt\right] &= O\left(\frac{1}{\sqrt{n\sigma}}\right). \numberthis \label{ptbound}
    \end{align*}
\end{restatable}
\begin{proof}[Proof of Lemma~\ref{varlemma}]
    We recall that $p_{T, Y = 1}(t)$ has the following form, which is a result of Proposition \ref{convprop}:
    \begin{align*}
        p_{T, Y=1}(t) = (\rho^{-1})'(t) \, \pi_Y(1) \, \E\left[p_{\xi}(\rho^{-1}(t) - h(X))\mid Y = 1\right]. \numberthis \label{pty1exp}
    \end{align*}
    Now from Cauchy--Schwarz, Jensen's inequality, Fubini's theorem, and a change of variable (in that order), we obtain:
    \begin{align*}
        \E_{\pi}\bigg[\int_0^1 \abs{p_{T, Y = 1}(t) - p_{\hat{T}, Y = 1}(t)} \ dt\bigg] &\le \E_{\pi}\left[\sqrt{\int_0^1 \abs{p_{T, Y = 1}(t) - p_{\hat{T}, Y = 1}(t)}^2 \ dt}\right] \\
        &\le \sqrt{\int_0^1 \E_{\pi} \left[\abs{p_{T, Y = 1}(t) - p_{\hat{T}, Y = 1}(t)}^2\right] \ dt} \\
        &= \Bigg(\int_0^1 \E_{\pi} \bigg[\bigg\vert (\rho^{-1})'(t) \,  \pi_Y(1) \, \E\left[p_{\xi}(\rho^{-1}(t) - h(X)) \mid Y=1 \right] \\
        &\quad \quad \quad - \frac{1}{n}\sum_{i = 1}^n (\rho^{-1})'(t) \, p_{\xi}(\rho^{-1}(t) - h(x_i)) \, \Ind_{y_i = 1}\bigg\vert^2 \bigg] \ dt \Bigg)^{1/2} \\
        &= \Bigg(\int_{-\infty}^{\infty} \E_{\pi} \bigg[ \bigg\vert \pi_Y(1) \, \E\left[p_{\xi}(u - h(X)) \mid Y=1 \right] \\
        &\quad \quad \quad - \frac{1}{n}\sum_{i = 1}^n p_{\xi}(u - h(x_i)) \, \Ind_{y_i = 1} \bigg\vert^2 \bigg] \ du\Bigg)^{1/2}. \numberthis \label{varchange}
    \end{align*}
    For simplicity, let us make the following definitions before moving forward:
    \begin{align*}
        f(u) &= \pi_Y(1) \, \E\left[p_{\xi}(u - h(X))\mid Y = 1\right], \numberthis \label{fudef} \\
        \hat{f}(u) &= \frac{1}{n}\sum_{i = 1}^n p_{\xi}(u - h(x_i)) \, \Ind_{y_i = 1}. \numberthis \label{fhatudef}
    \end{align*}
    We observe that $\E_{\pi}[\hat{f}(u)] = f(u)$ for every $u$, since the $(x_i, y_i)$ are i.i.d.~according to $\pi$. Now we can continue from \eqref{varchange}:
    \begin{align*}
        \E_{\pi}\bigg[\int_0^1 \abs{p_{T, Y = 1}(t) - p_{\hat{T}, Y = 1}(t)} \ dt\bigg] &\le \sqrt{\int_{-\infty}^{\infty} \E_{\pi} \left[\abs{f(u) - \hat{f}(u)}^2\right] \ du} \\
        &= \sqrt{\int_{-\infty}^{\infty} \Var_{\pi}(\hat{f}(u)) \ du} \\
        &\le \sqrt{\frac{\pi_Y(1)^2}{n} \int_{-\infty}^{\infty} \E\left[p_{\xi}(u - h(X))^2\mid Y = 1\right] \ du} \\
        &= \sqrt{\frac{\pi_Y(1)^2}{n\sigma} \E\left[\int_{-\infty}^{\infty} \frac{M}{\sigma} p_R\left(\frac{u - h(X)}{\sigma}\right) \ du\mid Y = 1 \right]} \\
        &= O\left(\frac{1}{\sqrt{n\sigma}}\right) \numberthis \label{sqrtnstep}
    \end{align*}
    where $M = \sup p_R$. The result for $\abs{p_{T}(t) - p_{\hat{T}}(t)}$ follows identically.
\end{proof}

With Lemma \ref{varlemma}, we prove our main estimation result.
\begin{theorem}\label{consistency}
    If $\xi = \sigma R$ for a random variable $R$ with bounded density, we have that
    \begin{align*}
        \E_{\pi}\big[ \abs{\SECE_{\pi, \xi}(h) - \SECE_{\hat{\pi}, \xi}(h)} \big] = O\big(1/\sqrt{n\sigma}\big), \numberthis \label{consistencyres}
    \end{align*}
    which implies $\SECE_{\hat{\pi}, \xi}(h) \to \SECE_{\pi, \xi}(h)$ in probability.
\end{theorem}
\begin{proof}
    Let $T$ be as in Lemma \ref{varlemma}, i.e. the population form of $\hat{T}$. Then we have by the triangle inequality:
    \begin{align*}
        &\E_{\pi}\big[\abs{\SECE_{\pi, \xi}(h) - \SECE_{\hat{\pi}, \xi}(h)}\big] \\
        &= \E_{\pi}\Bigg[\bigg\vert\int_0^1 \abs{\E[Y \mid T = t] - t} p_T(t) \ dt  - \int_0^1 \abs{\E[Y \mid \hat{T} = t] - t} p_{\hat{T}}(t) \ dt\bigg\vert\Bigg] \\
        &\le \E_{\pi}\bigg[\int_0^1 \abs{p_{T, Y = 1}(t) - p_{\hat{T}, Y = 1}(t)} \ dt  + \int_0^1 t\abs{p_{T}(t) - p_{\hat{T}}(t)} \ dt\bigg]. \numberthis \label{triangleineq}
    \end{align*}
    Noting that $t \le 1$ and applying Lemma \ref{varlemma} shows that the final line of \eqref{triangleineq} is $O(1/\sqrt{n\sigma})$, which is the desired result.
\end{proof}

The quantitative bound in Theorem \ref{consistency} shows that choosing $\sigma$ in practice is --- as intuition would suggest from the discussion of kernel regression --- similar to choosing the kernel bandwidth. In general, $\sigma$ can be thought of as a hyperparameter, but we will see in Section \ref{sec:experiments} that experiments are relatively insensitive to the choice of $\sigma$.

\subsection{Consequences for Estimation of ECE}\label{sec:eceest}
Nevertheless, for theoretical purposes, the scaling of $\sigma$ is important. The bound in Theorem \ref{consistency} suggests that $\sigma$ should be at least $\omega(1/n)$ to prevent the estimation error from exploding. 

It is then natural to ask what happens if we take $\sigma = \omega(1/n)$ but $\sigma \to 0$ as $n \to\infty$. By doing so --- analogously, once again, to existing results in the kernel density estimation literature --- we obtain that $\SECE_{\hat{\pi}, \xi_n}(h)$ actually becomes a consistent estimator of the true ECE, under appropriate conditions on the logit distribution $h(X)$. This is a non-trivial estimation result that we effectively get for free as a result of our framework, as shown in the short proof below.
\begin{restatable}{theorem}{ececonsist}\label{ececonsist}
    Suppose $\xi = \sigma R$ for a random variable $R$ with bounded, almost-everywhere continuous density, and let $\xi_n = \sigma_n \xi$ with $\sigma_n$ satisfying $\lim_{n \to \infty} \sigma_n = 0$ and $\sigma_n = \omega(1/n)$. 
    If the conditional distribution of $h(X)$ conditioned on $Y=y$ has an almost-everywhere continuous density for $y\in \{0,1\}$,
    then $\SECE_{\hat{\pi}, \xi_n}(h) \to \ECE_{\pi}(\rho \circ h)$ in probability.
\end{restatable}

\begin{proof}
First we claim that if $Z$ has an a.e. continuous density $p_Z$, then $Z+\xi_n\to Z$ in total variation. This follows by noting that $p_{\si \xi}(x) = \rc{\si}p_{\xi}\pf x{\si}$ is a sequence of good kernels (i.e.~an approximation to the identity)~\citep{stein2011fourier}, so that 
\[p_{Z+\xi_n}(x) = (p_Z *p_{\xi_n})(x) \to p_Z(x)\] at any continuity point of $p_Z$. Then $Z+\xi_n\tvto Z$ by Scheff\'e's Lemma.
 
Now for $y \in \{0,1\}$, by assumption and the above claim, conditioned on $Y=y$, we have $h(X)+\xi_n\tvto h(X)$.  
Hence $(Y,h(X)+\xi_n) \tvto (Y,h(X))$. 

By Lemma~\ref{l:tv-delta}, $|\SECE_{\pi, \xi_n}(h) - \ECE_\pi(\rh\circ h)|\to 0$ as $n\to \iy$. 
The result then follows from the triangle inequality and Theorem \ref{consistency}.
\end{proof}

\textbf{Comparison to existing estimation results.} To our knowledge, the only existing consistency results for ECE (or, more specifically, the $L^1$ ECE) are the works of \citet{zhang2020mix} and \citet{popordanoska2022consistent} that show consistency via adapting corresponding results for kernel density estimation. These results thus require assumptions such as H\"older-smooth and bounded or Lipschitz continuity of the density\footnote{Technically these assumptions were stated in terms of the conditional density of the predictor value $g(X)$ given $Y = y$, but they imply similar constraints on $h$ when considering $g = \rho \circ h$ with $\rho$ being the sigmoid function.} of $h(X)$ conditioned on $Y = y$, whereas we require only a.e. continuity.

\section{Experiments}\label{sec:experiments}
We now empirically verify that $\SECE_{\pi, \xi}$ behaves nicely even when $\ECE_{\pi}$ does not. We revisit the simple 2-point data distribution of Definition \ref{discretedist} in Section \ref{sec:synthdata}, and show that the discontinuity at $g(x) = 1/2$ leads to oscillatory behavior in $\ECE_{\BIN, \pi}$ (defined above in \eqref{binece}) as we change the number of bins, whereas $\SECE_{\hat{\pi}, \xi}$ remains effectively constant irrespective of the choice of variance for $\xi$. On the other hand, we also show that for image classification using a wide range of models, $\ECE_{\BIN, \pi}$ changes smoothly as we vary the number of bins, and the resulting estimates match up closely with both $\SECE_{\hat{\pi}, \xi}$ as well as the $\smECE$ of \citet{blasiok2023smooth}. For all experiments in this section, we take $\xi \sim \Gauss(0, \sigma^2)$. We consider using uniform noise in Appendix \ref{sec:kernelchoice}; the choice of $\xi$ does not impact our conclusions. All of the code used to generate the plots in this section can be found at: \url{https://github.com/2014mchidamb/how-flawed-is-ece}.

\subsection{Synthetic Data}\label{sec:synthdata}
We consider data drawn from a distribution $\pi$ as in Definition \ref{discretedist}, and the predictor $g(x) = \rho(\alpha x)$ where $\rho$ is the sigmoid function and $\alpha = 10^{-3}$. We construct $g$ in this way so that $g(-0.5) = 1/2 - \varepsilon$ and $g(0.5) = 1/2 + \varepsilon$, i.e. $g$ matches our discussion in Section \ref{sec:discrete}.

As we already know, $\ECE_{\pi}$ is discontinuous at the predictor which always predicts $1/2$. This immediately presents a problem for the estimation of $\ECE_{\pi}(g)$ via $\ECE_{\BIN, \pi}(g)$; indeed, one can readily compute (as done in \citep{blasiok2023unifying}) that $\ECE_{\BIN, \pi}(g)$ jumps between $\approx 0$ and $\approx 1/2$ depending on the parity of the number of bins used. We visualize this behavior in Figure \ref{synthplot}, where we plot $\ECE_{\BIN, \pi}(g)$ evaluated on 1000 samples from $\pi$ with the number of bins ranging from $1$ to $100$. 
\begin{figure}[h]
\centering
\includegraphics[width=0.5\columnwidth]{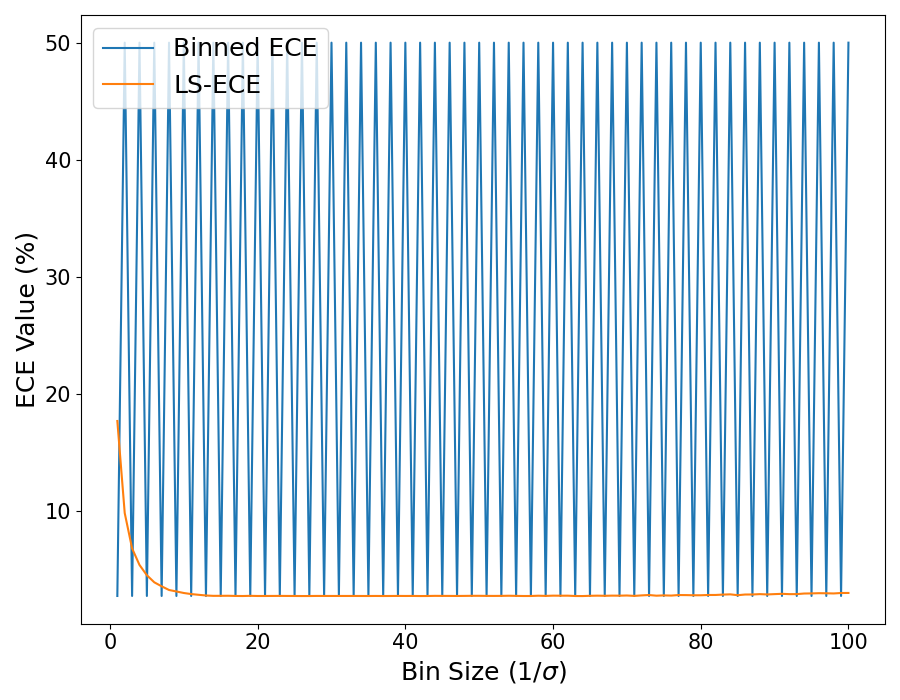} 
\caption{Comparison of $\ECE_{\BIN, \pi}$ (blue) and $\SECE_{\pi, \xi}$ (orange) over bins (and correspondingly, inverse scalings for $\xi$) ranging from 1 to 100 on the model and data setup of Section \ref{sec:synthdata}.}
\label{synthplot}
\end{figure}

Alongside $\ECE_{\BIN, \pi}(g)$, we also plot $\SECE_{\hat{\pi}, \xi}(h)$, where $h(x) = \alpha x$ and $\xi \sim \Gauss(0, \sigma^2)$. We let $\sigma$ be the inverse of the number of bins used for $\ECE_{\BIN, \pi}(g)$. The motivation for this choice of $\sigma$ comes from considering a uniform kernel (i.e. $\xi \sim \sigma \Uni([-1/2, 1/2])$), since in this case $\sigma$ corresponds to the bin size centered at each point. We estimate $\SECE_{\hat{\pi}, \xi}(h)$ via 10000 independent samples drawn from the distribution of $h(X) + \xi$, as discussed in Section \ref{sec:est}. We see that, outside the case of large $\sigma$ (i.e. $\sigma \gtrsim 1$), $\SECE_{\hat{\pi}, \xi}(h)$ remains effectively constant near zero and entirely avoids the oscillatory behavior of $\ECE_{\BIN, \pi}$.

\subsection{Image Classification}\label{sec:imageclass}
More importantly, we now check whether this disparity between $\ECE_{\BIN, \pi}(g)$ and $\SECE_{\hat{\pi}, \xi}(h)$ appears in settings of practical interest. We consider CIFAR-10, CIFAR-100 \citep{krizhevsky2009learning}, and ImageNet \citep{imagenet} and compare $\ECE_{\BIN, \pi}(g)$ using the bin numbers $\{1, 10, 20, ..., 100\}$ to both $\SECE_{\hat{\pi}, \xi}(h)$ with inversely proportional $\sigma$ and $\smECE$. We point out that $\smECE$ uses a particular choice of kernel bandwidth which we do not vary, so the $\smECE$ results are constant with respect to $\sigma$ for each model. As before, we estimate $\SECE_{\hat{\pi}, \xi}(h)$ using 10000 independent samples drawn from the distribution of $h(X) + \xi$.

Since all of the experiments in this section deal with multi-class classification, we use the top-class (or confidence calibration) formulations of ECE, LS-ECE, and $\smECE$ (see \eqref{tcecedef}). For LS-ECE, we construct the logit function $h(x)$ as $h(x) = \rho^{-1}(\max_i g^i(x))$, i.e. we apply the inverse sigmoid function to the maximum predicted probability.
\begin{figure}[h]
\centering
    \subfigure[CIFAR-10]{\includegraphics[width=0.48\columnwidth]{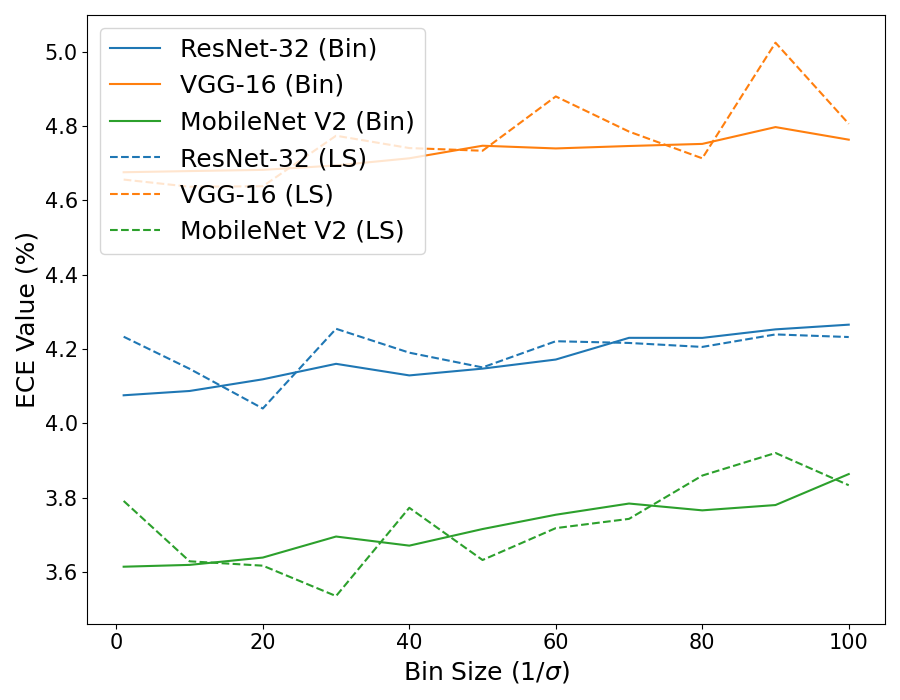}} 
    \subfigure[CIFAR-100]{\includegraphics[width=0.48\columnwidth]{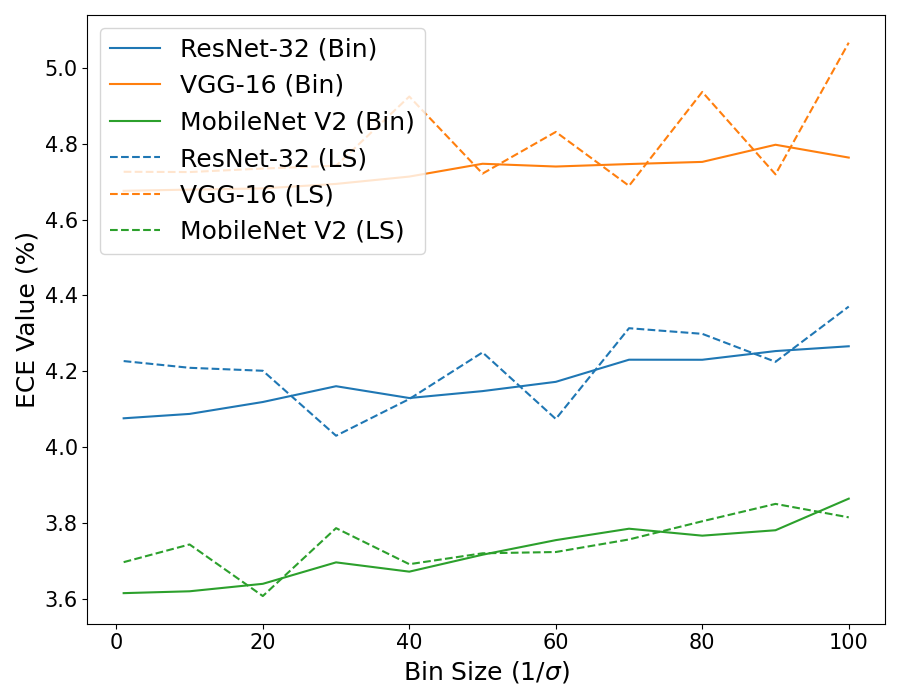}} 
\caption{Comparison of $\ECE_{\BIN, \pi}$ and $\SECE_{\pi, \xi}$ for different models on CIFAR datasets over bins/variance scalings ranging from 1 to 100. Solid lines correspond to $\ECE_{\BIN, \pi}$ and dashed lines correspond to $\SECE_{\pi, \xi}$.}
\label{cifarplots}
\end{figure}

\subsubsection{CIFAR Experiments}
For our CIFAR experiments, we use pretrained versions (due to Yaofo Chen) of ResNet-32 \citep{resnet}, VGG-16 \citep{simonyan2015deep}, and MobileNetV2 \citep{sandler2019mobilenetv2} available on TorchHub. Results for evaluating these models on the CIFAR-10 and CIFAR-100 test data are shown in Figure \ref{cifarplots}.

As can be seen, $\ECE_{\BIN, \pi}(g)$ stays nearly the same as we change the number of bins, and $\SECE_{\pi, \xi}(h)$ tracks it quite closely. (Although $\SECE_{\pi, \xi}(h)$ visually appears to exhibit more variance, the scale of this variance is small.) Furthermore, we note that the conclusions drawn from both ECE and LS-ECE about which model is best calibrated stay consistent across the choice of bin number/variance scaling.

\subsubsection{ImageNet Experiments}\label{sec:imagenet}
\begin{figure}[h]
\centering
\includegraphics[width=0.5\columnwidth]{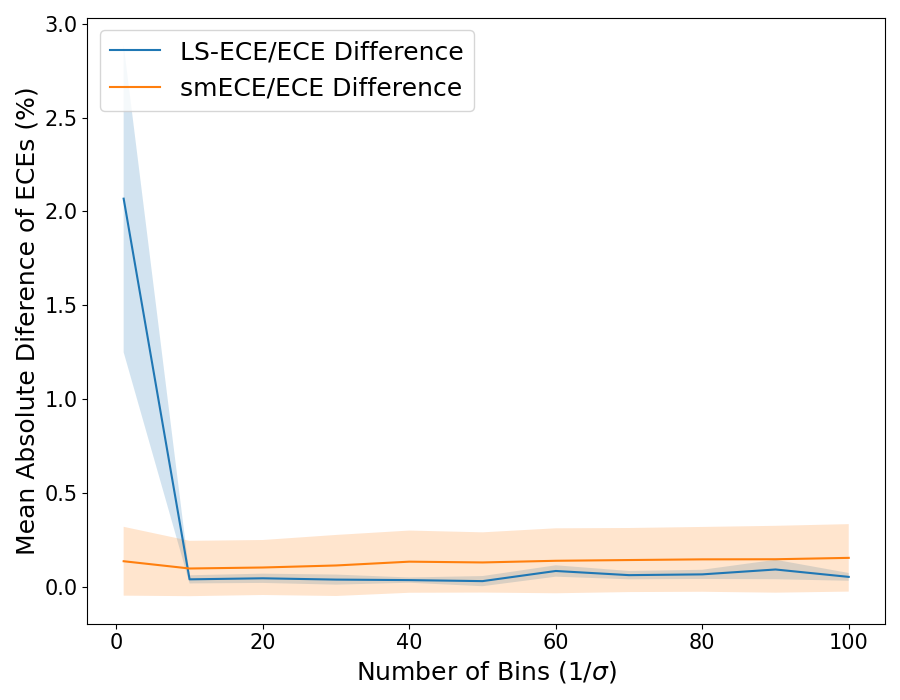} 
\caption{Mean absolute difference between ECE and LS-ECE, as well as ECE and $\smECE$, on ImageNet-1K-val over all models considered in Section \ref{sec:imagenet}, with one standard deviation error bounds marked using the shaded region.}
\label{imagenetplot}
\end{figure}
We repeat our CIFAR experimental setup for ImageNet, and consider a gamut of pretrained architectures: ResNet-18 and ResNet-50 \citep{resnet, resnetstrikes}, EfficientNet \citep{tan2019efficientnet}, MobileNetV3 \citep{howard2019searching}, Vision Transformer \citep{dosovitskiy2020vit}, and RegNet \citep{Radosavovic2020}. All of our models are obtained from the \texttt{timm} \citep{rw2019timm} library and were pretrained using the techniques described by \citet{resnetstrikes}. We evaluate all models on the ImageNet-1K validation data, and in Figure \ref{imagenetplot} we report the mean absolute difference (over all models) between ECE and LS-ECE, as well as ECE and $\smECE$, across bin numbers and choices of $\sigma$ respectively.

The ImageNet results further corroborate our CIFAR findings: ECE, LS-ECE, and $\smECE$ take near-identical values for all models considered, across the range of possible bin numbers and variances (although once again, $\smECE$ is not sensitive to these). Although this is by no means a comprehensive evaluation, the fact that the \textit{continuous} LS-ECE so closely tracks ECE in these experiments suggests that the theoretical pathologies of ECE may not pose a problem for assessing the calibration of real-world models in practice.

\section{Conclusion}
In summary, we have entirely characterized the discontinuities of ECE in a very general setting. We further used these continuity results to motivate the construction of LS-ECE, a continuous analogue of ECE that tracks it closely, and which in fact can be used to obtain a consistent estimator of ECE. As the results in this work are largely theoretical, a natural direction for future work would be a large-scale empirical validation of ECE results in the literature using the ideas we have presented.

\section*{Acknowledgments}

The work of M.C. is supported (via Rong Ge) by the National Science Foundation under grant numbers DMS-2031849 and CCF-1845171 (CAREER). The work of C.M. is supported by the National Science Foundation under grant number DMS-2103170 and by a grant from the Simons Foundation.  The work of S.R. is supported by the National Science Foundation under grant number DMS-2202959. Part of this work was conducted while the authors were attending the Random Theory 2023 workshop.

\bibliography{refs}
\bibliographystyle{refs}

\newpage
\appendix

\section{Impact of Noise Distribution on LS-ECE}\label{sec:kernelchoice}
Here we examine the effect of changing the noise distribution of $\xi$ on $\SECE_{\hat{\pi}, \xi}$. In particular, we contrast the choice of Gaussian noise to compactly supported noise, and consider instead $\xi \sim \sigma \Uni([-1/2, 1/2])$. Figure \ref{imagenetplotuniform} shows the results of redoing the experiments of Section \ref{sec:imagenet} with this choice of $\xi$. As can be seen, the results are essentially the same as Figure \ref{imagenetplot}.

\begin{figure}[h]
\centering
\includegraphics[width=0.5\columnwidth]{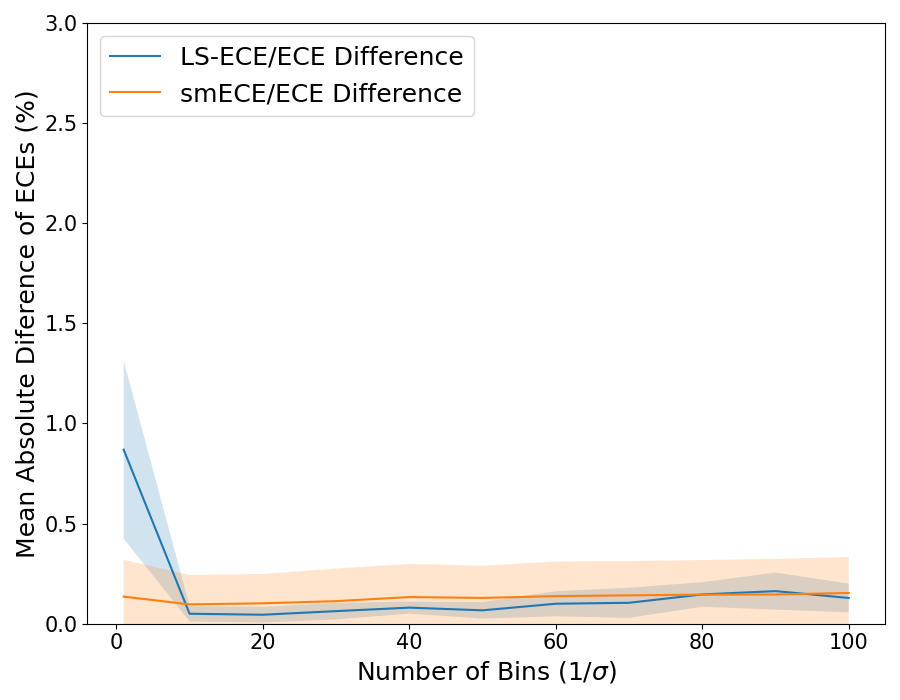} 
\caption{Mean absolute difference between ECE and LS-ECE (using uniform noise instead of Gaussian), as well as ECE and $\smECE$, on ImageNet-1K-val over all models considered in Section \ref{sec:imagenet}, with one standard deviation error bounds marked using the shaded region.}
\label{imagenetplotuniform}
\end{figure}

\end{document}